\def\year{2021}\relax
\documentclass[letterpaper]{article} % DO NOT CHANGE THIS
\usepackage{aaai21}  % DO NOT CHANGE THIS
\usepackage{times}  % DO NOT CHANGE THIS
\usepackage{helvet} % DO NOT CHANGE THIS
\usepackage{courier}  % DO NOT CHANGE THIS
\usepackage[hyphens]{url}  % DO NOT CHANGE THIS
\usepackage{graphicx} % DO NOT CHANGE THIS
\usepackage{natbib}  % DO NOT CHANGE THIS AND DO NOT ADD ANY OPTIONS TO IT
\usepackage{caption} % DO NOT CHANGE THIS AND DO NOT ADD ANY OPTIONS TO IT
\usepackage{amsmath}
\usepackage{amssymb}
\usepackage{amsthm}
\usepackage{arydshln}
\usepackage{listings}
\usepackage{subcaption}
\usepackage{multirow}
\usepackage{multicol}
\usepackage{tikz}
\usetikzlibrary{arrows,automata,shapes}
\newcommand{\tup}[1]{{\langle #1 \rangle}}
\usepackage{xcolor}

\newtheorem{definition}{Definition}
\newtheorem{theorem}{Theorem}

\title{Generalized Planning as Heuristic Search}
\author{
\textsuperscript{\rm 1}Javier Segovia-Aguas, \textsuperscript{\rm 2}Sergio Jim\'enez and \textsuperscript{\rm 1}Anders Jonsson\\ 
}

\affiliations{
\textsuperscript{\rm 1}Dept. Information and Communication Technologies, Universitat Pompeu Fabra, Spain\\
	\textsuperscript{\rm 2}VRAIN - Valencian Research Institute for Artificial Intelligence, Universitat Polit\`ecnica de Val\`encia, Spain\\
	
	javier.segovia@upf.edu, serjice@dsic.upv.es, anders.jonsson@upf.edu
}

\begin{document}

\maketitle

\begin{abstract}
Although {\em heuristic search} is one of the most successful approaches to classical planning, this planning paradigm does not apply straightforwardly to {\em Generalized Planning} (GP). {\em Planning as heuristic search} traditionally addresses the computation of sequential plans by searching in a grounded state-space. On the other hand GP aims at computing algorithm-like plans, that can branch and loop, and that generalize to a (possibly infinite) set of classical planning instances. This paper adapts the {\em planning as heuristic search} paradigm to the particularities of GP, and presents the first native heuristic search approach to GP. First, the paper defines a novel GP solution space that is independent of the number of planning instances in a GP problem, and the size of these instances. Second, the paper defines different evaluation and heuristic functions for guiding a combinatorial search in our GP solution space. Lastly the paper defines a GP algorithm, called Best-First Generalized Planning (BFGP), that implements a best-first search in the solution space guided by our evaluation/heuristic functions. 
\end{abstract}

\section{Introduction}
{\em Heuristic search} is one of the most successful approaches to classical planning~\cite{bonet2001planning,hoffmann2001ff,Helmert:FD:JAIR06,richter2010lama,lipovetzky2017best}. Unfortunately, it is not straightforward to adopt state-of-the-art search algorithms and heuristics from classical planning to {\em Generalized Planning} (GP). The {\em planning as heuristic search} approach traditionally addresses the computation of sequential plans with a grounded state-space search. Generalized planners reason however about the synthesis of algorithm-like solutions that, in addition to action sequences, contain branching and looping constructs. Furthermore, GP aims to synthesize solutions that generalize to a (possibly infinite) set of planning instances (where the domain of variables may be large), making the grounding of classical planners unfeasible~\cite{Winner03distill:learning,Levesque:GPlanning:IJCAI11,Zilberstein:Gplanning:icaps11,Giacomo:FSM:ICAPS13,bonet2019learning,illanes2019generalized}. 

This paper adapts the {\em planning as heuristic search} paradigm to the particularities of GP, and presents the first native heuristic search approach to GP. Given a GP problem, that comprises an input set of classical planning instances from a given domain,  our {\em GP as heuristic search} approach computes an algorithm-like plan that solves the full set of input instances. The contributions of the paper are three-fold:
\begin{enumerate}
	\item {\em A tractable solution space for GP}. We leverage the computational models of the {\em Random-Access Machine}~\cite{skiena1998algorithm} and the {\em Intel x86} FLAGS register~\cite{dandamudi2005installing} to define an innovative solution space that is independent of the number of input planning instances in a GP problem, and the size of these instances (i.e. the number of state variables and their domain size).
	\item {\em Grounding-free evaluation/heuristic functions} for GP. We define several evaluation and heuristic functions to guide a combinatorial search in our GP solution space.  Evaluating these functions does not require to ground states/actions in advance, so they allow to address GP problems where state variables have large domains (e.g. integers). 
	\item {\em A heuristic search algorithm for GP}. We present the {\sc BFGP} algorithm that implements a best-first search in our GP solution space. Experiments show that {\sc BFGP} significantly reduces the CPU-time required to compute and validate generalized plans, compared to the classical planning compilation approach to GP~\cite{segovia2019computing}.
\end{enumerate}

\section{Background}
This section introduces the necessary notation to define our heuristic search approach to generalized planning.
\subsection{Classical Planning}
Let $X$ be a set of {\em state variables}, each $x\in X$ with domain $D_x$. A {\em state} is a total assignment of values to the set of state variables, i.e.~$\tup{x_0=v_0, \ldots, x_N=v_N}$ such that for every $x_i\in X$ then $v_i\in D_{x_i}$. For a variable subset $X'\subseteq X$, let $D[X']=\times_{x\in X'} D_x$ denote its joint domain. The state space is then $S=D[X]$. Given a state $s\in S$ and a subset of variables $X'\subseteq X$, let $s_{|X'}=\tup{x_i=v_i}_{x_i\in X'}$ be the {\em projection} of $s$ onto $X'$ i.e.~the partial state defined by the values that $s$ assigns to the variables in $X'$. The {\em projection} of $s$ onto $X'$ defines the subset of states $\{s \mid s \in S, s_{|X'}\subseteq s\}$ that are consistent with the corresponding partial state.

Let $A$ be a set of deterministic {\em lifted actions}. A lifted action $a\in A$ has an associated set of variables $par(a)\subseteq X$, called {\em parameters}, and is characterized by two functions: an {\em applicability function} $\rho_a: D[par(a)] \rightarrow \{0,1\}$, and a {\em successor function} $\theta_a: D[par(a)]\rightarrow D[par(a)]$. Action $a$ is applicable in a given state $s$ iff $\rho_a(s_{|par(a)})$ equals $1$, and results in a {\em successor} state $s'=s\oplus a$, that is built replacing the values that $s$ assigns to variables in $par(a)$ with the values specified by $\theta_a(s_{|par(a)})$. Lifted actions generalize actions with {\em conditional effects}. These actions are common in GP since their state-dependent outcomes allow to adapt generalized plans to different classical planning instances. 

A {\em classical planning instance} is a tuple $P=\tup{X,A,I,G}$, where $X$ is a set of state variables, $A$ is a set of lifted actions, $I\in S$ is an initial state, and $G$ is a goal condition on the state variables that induces the subset of {\em goal states} %$S_G\subseteq S$ where 
$S_G = \{s \mid s \vDash G, s \in S\}$. Given $P$, a {\em plan} is an action sequence $\pi=\tup{a_1, \ldots, a_m}$ whose execution induces a {\em trajectory} $\tau=\tup{s_0, a_1, s_1, \ldots, a_m, s_m}$ such that, for each $1\leq i\leq m$, $a_i$ is applicable in $s_{i-1}$ and results in the successor state $s_i=s_{i-1}\oplus a_i$. A plan $\pi$ {\em solves} $P$ if the execution of $\pi$ in $s_0=I$ finishes in a goal state, i.e.~$s_m\in S_G$. We say $\pi$ is {\em optimal} if $|\pi|=m$ is minimal among the plans that solve $P$.

\subsection{Generalized Planning}
{\em Generalized planning} is an umbrella term that refers to more general notions of planning~\cite{jimenez2019review}. This work builds on top of the inductive formalism for GP, where a GP problem is defined as a set of classical planning instances with a common structure.

\begin{definition}[GP problem with shared state variables]
	\label{def:gp-problem}
A {\em GP problem} is a finite and non-empty set of $T$ classical planning instances $\mathcal{P}=\{P_1,\ldots,P_T\}$, where each classical planning instance $P_t\in\mathcal{P}$,  {\small $1\leq t \leq T$}, shares the same sets of state variables $X$ and lifted actions $A$, but may differ in the initial state and goals. Formally, $P_1=\tup{X,A,I_1,G_1}, \ldots, P_T=\tup{X,A,I_T,G_T}$.
\end{definition}

The representations of GP solutions range from {\em programs}~\cite{Winner03distill:learning,segovia2019computing} and {\it generalized policies}~\cite{Geffner:Gpolicies:AppliedI04,frances2019generalized}, to {\em finite state controllers}~\cite{Geffner:FSM:AAAI10,javi-Gplanning-IJCAI16}, {\em formal grammars} or {\em hierarchies} ~\cite{nau:shop2:JAIR03,ramirez2016heuristics,segovia2017generating}. Each representation has its own expressiveness capacity, as well as its own validation/computation complexity. We can however define a common condition under which a generalized plan is considered a solution to a GP problem. First, let us define $exec(\Pi,P)=\tup{a_1, \ldots, a_m}$ as the sequential plan produced by the execution of a generalized plan $\Pi$, on a classical planning instance $P=\tup{X,A,I,G}$.

\begin{definition}[GP solution]
	\label{def:gp-solution}
	A {\em generalized plan} $\Pi$ solves a GP problem $\mathcal{P}=\{P_1,\ldots,P_T\}$ iff, for every classical planning instance $P_t\in \mathcal{P}$, $ 1\leq t\leq T$, it holds that $exec(\Pi,P_t)$ solves $P_t$.
\end{definition}

We define the cost of a GP solution as the sum of the costs of the sequential plans that are produced for each of the classical planning instances in a GP problem. Formally,  $cost(\Pi,\mathcal{P})=\sum_{P_t\in\mathcal{P}}|exec(\Pi,P_t)|$. A GP solution $\Pi$ is optimal if $cost(\Pi,\mathcal{P})$ is minimal among  the generalized plans that solve $\mathcal{P}$.

{\bf Example}. Figure~\ref{fig:cp-example} shows the initial state and goal of two instances, $P_1=\tup{X,A,I_1,G_1}$ and $P_2=\tup{X,A,I_2,G_2}$, for sorting two six-element lists. Both instances share the set of state variables $X=\{x_1, x_2, x_3 ,x_4, x_5, x_6\}$ that encode six integers. They also share the set of lifted actions $A$, with $\frac{6\times 5}{2}$  $swap(x_i,x_j)$ actions that swap the content of two list positions {\small $i< j$}. An example solution plan for $P_1$ is  $\pi_1=\tup{swap(x_1,x_6), swap(x_2,x_3), swap(x_2,x_4)}$ while $\pi_2=\tup{swap(x_1,x_3), swap(x_4,x_6)}$ is a solution for $P_2$. The set $\mathcal{P}=\{P_1,P_2\}$ is an example GP problem. Sorting instances with different list sizes is made possible by including enough state variables to represent the elements of the longest list, and an extra state variable for the list length; e.g. $\mathsf{length}=l$ implies that list positions $[0,l-1]$ are valid.

\begin{figure}
	\centering
	\begin{tikzpicture}
	\draw[draw=black,step=0.5cm] (0.0,0.0) grid (3.0,0.5);
	\draw[draw=black,step=0.5cm] (4.0,0.0) grid (7.0,0.5);
	\draw[draw=black] (4.0,0.0) -- (4.0,0.5);
	\draw[draw=black,step=0.5cm] (0.0,1.0) grid (3.0,1.5);
	\draw[draw=black,step=0.5cm] (4.0,1.0) grid (7.0,1.5);
	\draw[draw=black] (4.0,1.0) -- (4.0,1.5);
	\draw[draw=black] (0.0,1.0) -- (3.0,1.0);
	\draw[draw=black] (4.0,1.0) -- (7.0,1.0);
	\draw[->,black] (3.1,0.25) -- (3.9,0.25);
	\draw[->,black] (3.1,1.25) -- (3.9,1.25);
	\node at (1.5,2.0) {\textbf{Initial State}};
	\node at (5.5,2.0) {\textbf{Goal State}};
	\node at (-0.5,0.25) {$P_2$};    
	\node at (0.25,0.25) {3};
	\node at (0.75,0.25) {2};
	\node at (1.25,0.25) {1};
	\node at (1.75,0.25) {6};
	\node at (2.25,0.25) {5};
	\node at (2.75,0.25) {4};
	\node at (4.25,0.25) {1};
	\node at (4.75,0.25) {2};
	\node at (5.25,0.25) {3};
	\node at (5.75,0.25) {4};
	\node at (6.25,0.25) {5};
	\node at (6.75,0.25) {6};
	\node at (-0.5,1.25) {$P_1$};        
	\node at (0.25,1.25) {6};
	\node at (0.75,1.25) {3};
	\node at (1.25,1.25) {4};
	\node at (1.75,1.25) {2};
	\node at (2.25,1.25) {5};
	\node at (2.75,1.25) {1};
	\node at (4.25,1.25) {1};
	\node at (4.75,1.25) {2};
	\node at (5.25,1.25) {3};
	\node at (5.75,1.25) {4};
	\node at (6.25,1.25) {5};
	\node at (6.75,1.25) {6};
	\end{tikzpicture}
	\caption{Classical planning instances for sorting the content of two six-element lists by swapping the list elements.}
	\label{fig:cp-example}
\end{figure}
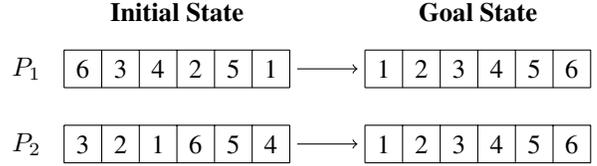

\subsection{Planning Programs}
\label{sec:planning-programs}
In this work we represent GP solutions as {\em planning programs}~\cite{segovia2019computing}. The execution of a planning program on a classical planning instance is a deterministic procedure that requires no variable matching; this is a useful property when implementing native heuristics for the space of GP solutions.

\subsubsection{Formalization.} A {\em planning program} is a sequence of $n$ instructions  $\Pi=\tup{w_0,\ldots,w_{n-1}}$, where each instruction $w_i\in \Pi$ is associated with a {\em program line} {\small $0\leq i< n$} and is either: 
\begin{itemize}
	\item A {\em planning action} $w_i\in A$.
	\item A {\em goto instruction} $w_i=\mathsf{go}(i',!y)$, where $i'$ is a program line $0\leq i'<i$ or $i+1<i'< n$, and $y$ is a proposition.
	\item A {\em termination instruction} $w_i=\mathsf{end}$. The last instruction of a program $\Pi$ is always  $w_{n-1}=\mathsf{end}$.
\end{itemize}

The execution model for a planning program  is a {\em program state} $(s,i)$, i.e.~a pair of a planning state $s\in S$ and program counter $0\leq i< n$. Given a program state $(s,i)$, the execution of a programmed instruction $w_i$ is defined as:
\begin{itemize}
	\item If $w_i\in A$, the new program state is $(s',i+1)$, where $s'=s\oplus w_i$ is the {\em successor} when applying $w_i$ in $s$.
	\item If $w_i=\mathsf{go}(i',!y)$, the new program state is $(s,i+1)$ if $y$ holds in $s$, and $(s,i')$ otherwise\footnote{We adopt the convention of jumping to line $i'$ whenever $y$ is {\em false}, inspired by jump instructions in the {\em Random-Access Machine} that jump when a register equals zero.}. Proposition $y$ can be the result of an arbitrary expression on state variables, e.g.~high-level state features~\cite{lotinac2016automatic}. 
	\item If $w_i=\mathsf{end}$, program execution terminates. 
\end{itemize}

To execute a planning program $\Pi$ on a planning instance $P=\tup{X,A,I,G}$, the initial program state is set to $(I,0)$, i.e.~the initial state of $P$ and the first program line of $\Pi$. A program $\Pi$ {\em solves} $P$ iff execution terminates in a program state $(s,i)$ that satisfies the goal condition, i.e.~$w_i=\mathsf{end}$ and $s\in S_G$. If a planning program {\em fails} to solve the instance, the only possible sources of failure are: (i) {\em incorrect program}, i.e.~execution terminates in a program state $(s,i)$ that does not satisfy the goal condition, i.e.~($w_i=\mathsf{end})\wedge (s\notin S_G)$; (ii) {\em inapplicable program}, i.e.~executing action $w_i\in A$ fails in program state $(s,i)$ since $w_i$ is not applicable in $s$; (iii) {\em infinite program}, i.e.~execution enters into an infinite loop that never reaches an $\mathsf{end}$ instruction.

\subsubsection{The space of planning programs.}
Given a maximum number of program lines $n$, where the last instruction is $w_{n-1}=\mathsf{end}$, and defined the propositions of goto instructions as $(x=v)$ atoms, where $x\in X$ and $v\in D_x$. The space of possible planning programs is compactly represented with three bit-vectors 
\begin{enumerate}
	\item The {\em action vector} of length $(n-1)\times |A|$, indicating whether action $a\in A$ appears on line $0\leq i< n-1$.
	\item The {\em transition vector} of length $(n-1)\times (n-2)$, indicating whether  $\mathsf{go}(i',*)$ appears on line $0\leq i< n-1$. 
	\item The {\em proposition vector} of length $(n-1)\times\sum_{x\in X}|D_x|$, indicating if $\mathsf{go}(*,!\tup{x=v})$ appears on line {\small $0\leq i< n-1$}.
\end{enumerate} 
A planning program is encoded as the concatenation of these three bit-vectors. The length of the resulting vector is: 
\begin{equation}
	(n-1) \left( |A|+ (n-2) + \sum_{x\in X}|D_x| \right).
	\label{eq:size}
\end{equation}

This encoding allows us to quantify the similarity of two planning programs (their {\em Hamming distance}) and more importantly, to implement a combinatorial search in the space of planning programs with $n$ lines. In fact, this is the encoding leveraged by the classical planning compilation approach to GP~\cite{segovia2019computing}. Equation~\ref{eq:size} reveals an important scalability limitation of this solution space that limits the applicability of the cited compilation to instances of small size: the number of planning programs with $n$ lines depends on the number, and domain size, of the state variables.

\section{A Tractable Solution Space for GP}
This section defines a scalable GP solution space of candidate {\em planning programs}, that is independent of the size of the input planning instances in a GP problem (i.e. the number of state variables and their domain size). Our approach is to extend the input classical planning instances of a GP problem with a general, and compact, feature language. This way the different classical planning instances in a given GP problem share a common set of features, even if their sets of states variables are actually different.

Our general and compact feature language leverages the computational model of the {\em Random-Access Machine} (RAM), an abstract computation machine that is polynomial-equivalent to a Turing machine~\cite{boolos2002computability}. The RAM machine enhances the {\em counter machine} with indirect memory addressing~\cite{minsky1961recursive}, and it allows augmented instruction sets and auxiliary {\em dedicated} registers, like the FLAGS register of the {\em Intel x86 microprocessors}~\cite{dandamudi2005installing}.

\subsection{Classical Planning with Pointers}

We extend the classical planning model with a RAM machine of $|Z|+2$ registers: $|Z|$ {\em pointers} that reference the original planning state variables, plus the {\em zero} and {\em carry} FLAGS (whose joint value will represent our common set of features for GP). Given a classical planning instance $P=\tup{X,A,I,G}$, the {\em extended classical planning instance} with a RAM machine of $|Z|+2$ registers is defined as $P_Z'=\tup{X_Z',A_Z',I_Z',G}$, where: 
\begin{itemize}
	\item The new set of {\bf state variables}
	$X_Z'$ comprises:
	\begin{itemize}
		\item The state variables $X$ of the original planning instance.
		\item Two {\em Boolean variables} $Y=\{y_z,y_c\}$, that play the role of the {\em zero} and  {\em carry} FLAGS register.
		\item The {\em pointers} $Z$, a non-empty set of extra state variables with finite domain {\small $[0,\ldots,|X|-1]$}.
	\end{itemize}
	\item The new set of actions $A_Z'$ includes: 
		\begin{itemize}
		\item The {\bf planning actions} $A'$ that result from abstracting each original planning action $a\in A$ into its corresponding action scheme. The abstraction is a two-step procedure that requires at least as many pointers as the largest arity of an action in $A$: (i), each action parameter is replaced with a pointer in $Z$ (ii), its applicability/successor function is rewritten to refer to the state variables pointed by the action parameters. For instance, the $\frac{6\times 5}{2}$  {\tt\small swap($x_i$,$x_j$)} actions from the example of Figure~\ref{fig:cp-example}, are abstracted into the  action scheme $\{${\tt\small swap($z_1$,$z_2$)} $| \; z_1,z_2 \in Z\}$ using two pointers, ~$Z=\{z_1,z_2\}$.  
	
		\item The {\bf RAM actions} that implement the following set of RAM instructions $\{{\tt\small inc}(z_1)$, ${\tt\small dec}(z_1)$, ${\tt\small cmp}(z_1,z_2)$, ${\tt\small cmp}(*z_1,*z_2)$, ${\tt\small set}(z_1,z_2)$ $| \; z_1,z_2 \in Z\}$ over the pointers in Z. Respectively, these RAM instructions increment/decrement a pointer, compare two pointers (or their content), and set the value of a pointer $z_1$ to another pointer $z_2$. Each RAM action also updates the {\sc flags} $Y=\{y_z,y_c\}$, according to the result of the corresponding RAM instruction (denoted here by {\tt\small res}):
\begin{small}
\begin{align*}
 y_z &:= ( res == 0 ),\\
 y_c &:= ( res > 0 ),\\
 inc(z_1) &\implies res := z_1 + 1,\\
 dec(z_1) &\implies res := z_1 - 1,\\
 cmp(z_1,z_2) &\implies res := z_1 - z_2,\\
 cmp(*z_1,*z_2) &\implies res := *z_1 - *z_2,\\ 
 set(z_1,z_2) &\implies res := z_2.
\end{align*}
\end{small}
\end{itemize}	
The number of actions of a  classical planning instance, extended with a RAM machine of $|Z|+2$ registers, is
\begin{equation}
	|A_Z'|=2|Z|^2+|A'|.
	\label{eq:nactions}
\end{equation}
The {\em increment}/{\em decrement} instructions induce $2|Z|$ actions, the {\em set} instructions over pointers induce $|Z|^2-|Z|$ actions, and {\em comparison} instructions induce $|Z|^2-|Z|$ actions; they can compare pointers or their content, but we only consider a single parameter ordering for symmetry breaking, i.e.~{\tt\small cmp($z_1$,$z_2$)} but not {\tt\small cmp($z_2$,$z_1$)}. 

	\item The new {\bf initial state} $I_Z'$ is the initial state of the original planning instance, but extended with all {\em pointers} set to zero and {\em FLAGS} set to {\tt\small False}. The {\bf goals} are the same as those of the original planning instance.
\end{itemize}

\begin{theorem}
Given a classical planning instance $P$, its extension $P'_Z$, with a RAM machine of $|Z|+ 2$ registers, preserves the completeness and soundness of $P$.
\end{theorem}

\begin{proof}
$\Rightarrow$: Let $\pi=\tup{a_1,\ldots,a_m}$ be a plan that solves $P$. For each action $a_i$ with parameter set $par(a_i)\subseteq X$, $A'$ contains an action scheme $a_i'$ that replaces the parameters in $par(a_i)$ with pointers in $Z$. For each such pointer $z\in Z$, repeatedly apply RAM actions {\tt\small inc($z$)}, or {\tt\small dec($z$)}, until it references the associated state variable in $par(a_i)$, and then apply $a_i'$. The resulting plan $\pi'$ has exactly the same effect as $\pi$ on the original planning state variables in $X$, and since the goal condition of $P'_Z$ is the same as that of $P$, it follows that $\pi'$ solves $P'_Z$.

$\Leftarrow$: Let $\pi'=\tup{a_1',\ldots,a_m'}$ be a plan that solves $P'_Z$. Identify each action in $A'$ among those of $\pi'$, and execute $\pi'$ to identify the assignment of variables to pointers when applying each action in $A'$. Construct a plan $\pi$ corresponding to the subsequence of actions in $A'$ from $\pi'$, replacing each action scheme $a_i'\in A'$ by an original action $a_i\in A$ and choosing as parameters of $a_i$ the state variables referenced by the pointers of $a_i'$ at the moment of execution. Hence $a_i$ has the same effect as $a_i'$ on the state variables in $X$, implying that $\pi$ has the same effect as $\pi'$ on $X$. Since the goal condition of $P$ is the same as that of $P'_Z$, it follows that $\pi$ solves $P$.
\end{proof}

\subsection{Generalized Planning with Pointers}
Our extension of classical planning instances with a RAM machine of $|Z|+ 2$ registers allows us to define the following minimalist feature language.

\begin{definition}[The feature language]
	\label{def:gp-problem}
We define the feature language $\mathcal{L}=\{\neg y_z\wedge \neg y_c, y_z\wedge \neg y_c, \neg y_z\wedge y_c,y_z\wedge y_c\}$ of the four possible joint values for the pair of Boolean variables $Y=\{y_z,y_c\}$.  
\end{definition}

We say that our feature language $\mathcal{L}$ is general, because it is independent of the number (and domain size) of the planning state variables. Different classical planning instances, that are extended with a RAM machine of $|Z|+ 2$ registers, share the features language $\mathcal{L}$ even if they actually have different sets of planning state variables. In this regard, we extend our previous GP problem definition to a set of classical planning instances that, as in ~\cite{hu2011generalized}, they share the same features and actions, but may actually have different state variables. 

\begin{definition}[GP problem with shared features]
	\label{def:gp-problem}
A {\em GP problem with shared features} is a finite and non-empty set of $T$ classical planning instances $\mathcal{P}=\{P_1,\ldots,P_T\}$, where  instances share the same set of actions $A_Z'$, but may differ in the state variables, initial state, and goals. Formally, $P_1=\tup{X'_{1_Z},A_Z',I'_{1_Z},G_1}, \ldots, P_T=\tup{X'_{T_Z},A_Z',I'_{T_Z},G_T}$.
\end{definition}

We leverage our minimalist feature language $\mathcal{L}$ to define a tractable solution space for GP; we represent GP solutions as {\em planning programs} where goto instructions are exclusively conditioned on a feature in $\mathcal{L}$.  Limiting the conditions of goto instructions to any of the four features in $\mathcal{L}$ reduces the number of planning programs with $n$ lines, specially when state variables have large domains; the {\em proposition vector} required to encode a planning program becomes now a vector of only $(n-1)\times 4$ bits, one bit for each of the four features in $\mathcal{L}$. Equation~\ref{eq:size} simplifies then to:
\begin{equation}
	(n-1) \left( |A_Z'| + (n-2) + 4 \right).
	\label{eq:size2}
\end{equation}

As illustrated by Equation~\ref{eq:size2}, the size of this new solution space for GP is independent of the number (and domain size) of the planning state variables. Equation~\ref{eq:nactions}  already showed that $|A_Z'|$ no longer grows with the number/ domain of the planning state variables, and that instead, it grows  with the number of pointers $|Z|$. This means that this GP solution space  scales to GP problems where state variables have large domains (e.g. integers). The solution space is still incomplete in the sense that either the given bound $n$ on the maximum number of program lines, or the maximum number of pointers available $|Z|$, may be too small to accommodate a solution to a given GP problem. 

\subsubsection{Example.} Figure~\ref{fig:lists} shows two example planning programs found by our {\sc BFGP} algorithm for GP: (left) reversing a list; and (right) sorting a list. In both cases the programs generalize, no matter the list length or its content.  In {\em reverse}, line 0 sets the pointer $j$ to the last list element. Then, line 1 swaps the element pointed by $i$ (initially set to zero) and the element pointed by $j$, pointer $j$ is decremented, pointer $i$ is incremented, and this sequence of instructions is repeated until the condition on line 5 becomes false, i.e, when pointers $i$ and $j$ meet, which means that reversing the list is finished. In {\em sorting}, pointers $i$ and $j$ are used for inner (lines 5-7) and outer (lines 8-11)  loops respectively, and $min$ to point to the minimum value in the inner loop (lines 3-4); $\neg y_z\wedge \neg y_c$ on line $2$ represents whether the content of $j$ is less than the content of $min$, while $y_z\wedge\neg y_c$ on line 7 represents whether $j==\mathsf{length}$ (and respectively ~$i==\mathsf{length}$ on line 11).

\begin{figure}
\small
\hspace{-2.35cm}
		\begin{subfigure}[t]{.35\columnwidth}
			\begin{lstlisting}[mathescape]
            0. set(j,tail)
            1. swap(*i,*j)
            2. dec(j)
            3. inc(i)
            4. cmp(j,i)
            5. goto(1,$\neg (\neg y_z\wedge\neg y_c)$)
            6. end			
			\end{lstlisting}
		\end{subfigure}	
\hspace{-1.3cm}		
		\begin{subfigure}[t]{.25\columnwidth}
			\begin{lstlisting}[mathescape]
		 	 0. set(min,i)
			 1. cmp(*j,*min)
			 2. goto(5,$\neg (\neg y_z\wedge \neg y_c)$)
			 3. set(min, j)
			 4. swap(*i,*min)
			 5. inc(j)
			 6. cmp(length,j)
			 7. goto(1,$\neg (y_z\wedge\neg y_c)$)
			 8. inc(i) 
			 9. set(j,i)
			10. cmp(length,i)
			11. goto(0,$\neg (y_z\wedge\neg y_c)$)    
			12. end
			\end{lstlisting}
		\end{subfigure}	
	\caption{{\em Generalized plans}: (left) for reversing a list; (right) for sorting a list with the {\em selection-sort} algorithm.}
	\label{fig:lists}
\end{figure}

\section{Generalized Planning as Heuristic Search}
This section describes our heuristic search approach to GP; a best-first search in our space of candidate planning programs with at most $n$ program lines, and $|Z|$ pointers available. First, we describe the evaluation and heuristic functions and then, we describe the search algorithm.

\subsection{Evaluation and Heuristic Functions}
\label{subsec:error}
We exploit two different sources of information to guide the search in the space of candidate planning programs:
\begin{itemize}
	\item {\em The program structure}. These are evaluation functions, computed in linear time, traversing the bit-vector representation of a planning program $\Pi$. 
		\begin{itemize}
			%\item$f_1(\Pi)$, the maximum number of nested loops in $\Pi$. 
			\item$f_1(\Pi)$, the number of {\em goto} instructions in $\Pi$.
			\item$f_2(\Pi)$, the number of {\em undefined} program lines in $\Pi$.
			%\item$h_6(\Pi)$, the max number of times that an action appears in $\Pi$.
			\item$f_3(\Pi)$, the number of repeated actions in $\Pi$.
		\end{itemize}
	
	\item {\em The empirical performance of the program}. These functions assess the performance of a planning program $\Pi$ on a GP problem $\mathcal{P}$,  executing $\Pi$ on each of the classical planning instances $P_t\in\mathcal{P}$, {\small $1\leq t \leq T$}:
		\begin{itemize}
			\item$h_4(\Pi,\mathcal{P})=n-PC^{MAX}$, where $PC^{MAX}$ is the maximum program line that is eventually reached after executing $\Pi$ on all the classical planning instances in $\mathcal{P}$. 
			\item $h_5(\Pi,\mathcal{P})=\sum_{P_t\in \mathcal{P}} h_5(\Pi,P_t)$, where \[h_5(\Pi,P_t)=\sum_{x\in X_t} (v_x-G_t(x))^2.\] Here, $v_x\in D_x$ is the value eventually reached, for the state variable $x\in X_t$, after executing $\Pi$ on the classical planning instance $P_t\in\mathcal{P}$, and $G_t(x)$ is the value for this same variable as specified in the goals of $P_t$. Computing $h_5(\Pi,P_t)$ requires then that the goal condition of the classical planning instance $P_t$ is specified as a partial state. Note also that for Boolean variables the squared difference becomes a simple difference.
			\item $f_6(\Pi,\mathcal{P}) = \sum_{P_t\in{\cal P}} |exec(\Pi,P_t)|$, is the cost of a GP solution, as defined in the {\em Background} section of the paper. That is $exec(\Pi,P_t)$ is the sequence of actions induced from executing the planning program $\Pi$ on the planning instance $P_t$. 
		\end{itemize}
\end{itemize}
All these functions are {\em cost functions} (i.e.~smaller values are preferred). Functions $h_4(\Pi,\mathcal{P})$ and $h_5(\Pi,\mathcal{P})$ are cost-to-go {\em heuristics}; they provide an estimate on how far a program is from solving the given GP problem.  Functions $h_4(\Pi,\mathcal{P})$, $h_5(\Pi,\mathcal{P})$, and $f_6(\Pi,\mathcal{P})$ aggregate several costs that could be expressed as a combination of different functions, e.g.~{\em sum}, {\em max}, average, weighted average, etc.  

{\bf Example}. We illustrate how our evaluation/heuristic functions work on the program $\Pi= \;${\tt\small 0. swap(*i, *j)}\linebreak {\tt\small 1. inc(i)} {\tt\small 2. dec(j)} {\tt\small 3\ldots} {\tt\small 4\ldots } {\tt\small 5. end}, where only lines 0-2 are specified because lines 3-4 are not programmed yet. The value of the evaluation functions for this partially specified program is $f_1(\Pi)=0$, $f_2(\Pi)=5-3=2$, $f_3(\Pi)=0$. Given the GP problem $\mathcal{P}=\{P_1,P_2\}$ that comprises the two classical planning instances illustrated in Figure~\ref{fig:cp-example}, and pointers $i$ and $j$ starting at the first and last memory register, respectively, we can compute $h_4$ and $h_5$ to evaluate how far $\Pi$ is from solving the GP problem of sorting lists, and the accumulated cost $f_6$. In this case $h_4(\Pi,\mathcal{P})=5-3=2$, $h_5(\Pi,\mathcal{P})=32$ and  $f_6(\Pi,\mathcal{P})=3+3=6$.

\subsection{Best First Search for Generalized Planning}
Here we provide the implementation details of our {\sc BFGP} algorithm for generalized planning. This algorithm implements a {\em Best-First Search} (BFS) in the space of planning programs with $n$ program lines, and a RAM machine of $|Z|+ 2$ registers. 

The {\sc BFGP} algorithm starts a BFS with an empty planning program, that is represented by the concatenation of the bit-vectors defined in the previous section, with all their bits set to {\tt False}. To generate a tractable set of successor nodes, child nodes in the search tree are restricted to programs that result from programming the bits that correspond to the $PC^{MAX}$ line (i.e. the maximum line reached after executing the current program on the classical planning instances in $\mathcal{P}$). The planning program of a search node is then at Hamming distance 1 from its parent, when programming a planning action, or at Hamming distance 2, when programming a goto instruction. This procedure for successor generation guarantees that duplicate successors are not generated.  

{\sc BFGP} is a {\em frontier search} algorithm, meaning that, to reduce memory requirements, {\sc BFGP} stores only the open list of generated nodes, but not the closed list of expanded nodes~\cite{korf2005frontier}. BFS sequentially expands the best node in a priority queue (aka {\em open list}) sorted by an evaluation/heuristic function. When evaluating the {\em empirical performance} of a program, e.g. computing heuristics $h_i(\Pi,\mathcal{P})$, if the execution of $\Pi$ on a given instance $P_t\in \mathcal{P}$ fails, this means that the search node corresponding to the planning program $\Pi$ is a dead-end, and hence it is not added to the open list. If the planning program $\Pi$ solves all the instances $P_t\in \mathcal{P}$, then search ends, and $\Pi$ is a valid solution for the GP problem $\mathcal{P}$. {\sc BFGP} can compute optimal solutions when run in {\em anytime mode}. In this case we can use $f_6(\Pi,\mathcal{P})$ to rank GP solutions (e.g. to prefer a sorting program with smaller computational complexity).

\begin{table*}[]
	\centering
	\begin{small}    
		\begin{tabular}{|l|c||c|c|c|c||c|c|c|c||c|c|c|c|} \hline
			\multirow{2}{*}{\textbf{Domain}} & \multirow{2}{*}{$n,|Z|$} & \multicolumn{4}{|c||}{$f_1$} & \multicolumn{4}{|c||}{$f_2$} & \multicolumn{4}{|c|}{$f_3$}  \\\cline{3-14}
			& & Time & Mem. & Exp. & Eval.  & Time & Mem. & Exp. & Eval.  & Time & Mem. & Exp. & Eval.  \\\hline
			T. Sum & 5, 2 & 0.24 & 4.2 & 4.8K & 5.8K & 0.30 & {\bf 3.8} & 6.4K & 6.4K & 0.13 & 3.9 & 2.2K & 2.8K \\
			Corridor & 7, 2 & 3.04 & 6.6 & 12.4K & 26.7K & {\bf 0.41} & {\bf 3.8} & {\bf 2.2K} & {\bf 2.2K} & 3.66 & 4.2 & 24.2K & 25.9K \\
			Reverse & 7, 3 & 82 & 61 & 0.28M & 0.57M & 181 & {\bf 4.0} & 0.95M & 0.95M & 170 & 23 & 0.75M & 0.84M \\
			Select & 7, 3 & 198 & 110 & 0.83M & 1.10M & 27 & {\bf 3.9} & 0.12M & 0.12M & 76.49 & 11 & 0.34M & 0.38M \\
			Find & 7, 3 & 195 & 175 & 0.50M & 1.36M & 271 & {\bf 4.0} & 1.46M & 1.46M & {\bf 86} & 12 & {\bf 0.41M} & {\bf 0.45M} \\
			Fibonacci & 8, 3 & {\bf 496} & 922 & {\bf 2.48M} & {\bf 6.79M} & 1,082 & {\bf 3.9} & 11.8M & 11.8M & TO & - & - & - \\
			Gripper & 8, 4 & TO & - & - & - & 3,439 & {\bf 4.1} & 19.9M & 19.9M  & TO & - & - & - \\
			Sorting & 9, 3 & TO & - & - & - & TO & - & - & -  & {\bf 3,143} & {\bf 711} & {\bf 19.5M} & {\bf 22.9M} \\\hline
		    \multicolumn{2}{|c||}{Average} & 162 & 213 & 0.68M & 1.64M & 714 & 3.9 & 4.89M & 4.89M & 580 & 128 & 3.51M & 4.09M \\\hline
		\end{tabular}
		\begin{tabular}{|l|c||c|c|c|c||c|c|c|c||c|c|c|c|} \hline
			\multirow{2}{*}{\textbf{Domain}} & \multirow{2}{*}{$n,|Z|$} &  \multicolumn{4}{|c||}{$h_4$} & \multicolumn{4}{|c||}{$h_5$} & \multicolumn{4}{|c|}{$f_6$} \\\cline{3-14}
			& & Time & Mem. & Exp. & Eval.  & Time & Mem. & Exp. & Eval.  & Time & Mem. & Exp. & Eval.  \\\hline
			T. Sum & 5, 2 & 0.10 & {\bf 3.8} & 1.6K & 1.6K & {\bf 0.09} & 3.9 & {\bf 1.2K} & {\bf 1.4K} & 0.45 & 4.8 & 7.4K & 7.4K \\
			Corridor & 7, 2 & 1.09 & 3.9 & 5.3K & 5.4K & 5.29 & 4.1 & 30.3K  & 31.3K & 6.16 & 7.6 & 35.3K & 35.3K \\
			Reverse & 7, 3 & 205 & 4.2 & 0.88M & 0.88M & {\bf 1.46} & 4.2 & {\bf 4.9K} & {\bf 6.3K} & 369 & 230 & 1.57M & 1.65M \\
			Select & 7, 3 & {\bf 0.80} & {\bf 3.9} & {\bf 3.0K} & {\bf 4.2K} & 94 & 5.7 & 0.34M & 0.35M & 255 & 155 & 1.06M & 1.14M \\
			Find & 7, 3 & 415 & 4.4 & 1.76M & 1.76M & 140 & 7.0 & 0.58M & 0.59M & 423 & 244 & 1.76M & 1.77M \\
			Fibonacci & 8, 3 & TO & - & - & - & 1,500 & 120 & 11.3M & 11.8M & TO & - & - & - \\
			Gripper & 8, 4  & TO & - & - & - & {\bf 83} & 5.5 & {\bf 0.34M} & {\bf 0.35M} & TO & - & - & - \\
			Sorting & 9, 3 & TO & - & - & - & TO & - & - & -  & TO & - & - & - \\\hline
		    \multicolumn{2}{|c||}{Average} & 125 & 4.0 & 0.53M & 0.53M & 260 & 21 & 1.80M & 1.88M & 211 & 128 & 0.89M & 0.92M \\\hline
		\end{tabular}
	\end{small}
	\caption{ We report the number of program lines $n$, and pointers $|Z|$ per domain, and for each evaluation/heuristic function, CPU (secs), memory peak (MBs), and the numbers of expanded and evaluated nodes. TO stands for Time-Out ($>$1h of CPU).}
	\label{tab:heuristics}
\end{table*}

\section{Evaluation}
This section evaluates the empirical performance of our {\em GP as heuristic search} approach. All experiments are performed in an Ubuntu 20.04 LTS, with AMD® Ryzen 7 3700x 8-core processor $\times$ 16 and 32GB of RAM. We make fully available the compilation source code, evaluation scripts, and used
benchmarks \cite{segovia_aguas_javier_2021_4600833},
%\footnote{At this public repository, \small\it\url{https://github.com/rleap-project/best-first-generalized-planning}}, 
so any experimental data reported in the paper is fully reproducible.

\subsection{Benchmarks} 
We report results in eight domains of two kinds ({\em numerical series} and {\em memory manipulation}). Numerical series compute the $n^{th}$ term of a math sequence. {\em Memory manipulation} domains include {\em Reverse} for reversing the content of a list, {\em Select} for choosing the minimum value in the list, {\em Find} for counting the number of occurrences of a specific value in a list, {\em Corridor} to move to any arbitrary location in a corridor starting from another arbitrary location, and {\em Gripper} to move all balls from one room to another. 
In the numerical series domains, the planning instances of a GP problem represent the computation of the first ten terms of the series. In the memory manipulation domains, the planning instances have initial random content, with list sizes ranging from two to twenty-one (besides {\em gripper} where all balls are initially in room A). The results of arithmetical operations is bounded to $10^2$ in the synthesis of GP solutions, and to $10^9$ in the validation of GP solutions.

All domains include the base RAM instructions $\{{\tt\small inc}(z_1)$, ${\tt\small dec}(z_1)$, ${\tt\small cmp}(z_1,z_2)$, ${\tt\small cmp}(*z_1,*z_2)$, ${\tt\small set}(z_1,z_2)$ $| \; z_1,z_2 \in Z\}$. In addition, each domain also contains regular planning action schemes, that do not affect the FLAGS.
\begin{itemize}
    \item {\em Triangular Sum} and {\em Fibonacci} include the action scheme $add(*z_1,*z_2)$ for adding the referenced value of a pointer to the content of the other pointer.
    \item {\em Select} and {\em Find} only require the base RAM instructions. {\em Reverse} and {\em Sorting} include also the $swap(*z_1,*z_2)$ action scheme. {\em Corridor} includes $left(*z_1)$ and $right(*z_1)$ schemes for moving left and right in the corridor, respectively, and {\em Gripper} includes the action schemes $pick(*z_1)$, $drop(*z_1)$, $moveAB()$, and $moveBA()$. 
\end{itemize}

\subsection{Synthesis of GP Solutions}
Table~\ref{tab:heuristics} summarizes the results of {\sc BFGP} with our six different evaluation/heuristic functions (best results in bold): i) $f_2$ and $h_5$ exhibited the best coverage; ii) there is no clear dominance of a {\em structure} evaluation function, $f_2$ has the best memory consumption while $f_3$ is the only structural function that solves {\em Sorting}; iii) the {\em performance}-based function $h_5$ dominates $h_4$ and $f_6$. 

Interestingly, the base performance of {\sc BFGP} with a single evaluation/heuristic function is improved combining both structural and cost-to-go information; we can guide the search of {\sc BFGP} with a cost-to-go heuristic function and break ties with a structural evaluation function, and vice versa. Table~\ref{tab:h-combined} shows the performance of $BFGP(f_1,h_5)$ and its reversed configuration $BFGP(h_5,f_1)$ which actually resulted in the overall best configuration solving all domains. %We provide supplementary material with expanded experimental results in which the performance of the heuristic function $h_5$ is evaluated in combination with all the possible structural functions (namely $f_1$, $f_2$ and $f_3$).

\begin{table}[]
\begin{small} 
    \centering
    \begin{tabular}{|l||c|c|c|c|} \hline%\cline{2-9}
        \multirow{2}{*}{\bf Dom.} & \multicolumn{4}{|c|}{BFGP$(f_1,h_5)$ / BFGP$(h_5,f_1)$}\\\cline{2-5} % & \multicolumn{4}{|c|}{BFGP$(h_5,f_1)$} \\\cline{2-9}
          & T.  & M. & Exp. & Eval.  \\\hline
            T. Sum & 0.2/{\bf 0.1} & 4.3/{\bf 3.8} & 2.8K/{\bf 1.1K} & 4.8K/{\bf 1.4K}\\
			Corr. & {\bf 3.2}/4.5 & 6.6/{\bf 5.9} & {\bf 6.5K}/26.0K & {\bf 21.6K}/27.5K \\
			Rev.  & 63/{\bf 1.4} & 52/{\bf 4.7} & 81.9K/{\bf 3.7K} & 0.3M/{\bf 7.7K}\\
			Sel.  & 203/{\bf 80}  & 110/{\bf 7.0} & 0.6M/{\bf 0.3M} & 0.9M/{\bf 0.3M} \\
			Find  & 313/{\bf 162} & 176/{\bf 14} & 0.9M/{\bf 0.7M} & 1.5M/{\bf 0.7M} \\
			Fibo.  & 528/{\bf 22} & 828/{\bf 32} & 1.4M/{\bf 75K} & 5.3M/{\bf 0.2M} \\
			Grip.  & TO/{\bf 6.9} & -/{\bf 10} & -/{\bf 5.8K} & -/{\bf 37.3K}   \\
			Sort.  & TO/{\bf 713} & -/{\bf 730} & -/{\bf 4.4M} & -/{\bf 4.5M}  \\\hline
    \end{tabular} 
    \caption{For each domain we report, CPU time (secs), memory peak (MBs), num. of expanded and evaluated nodes. TO means time-out ($>$ 1h of CPU). Best results in bold.}
    \label{tab:h-combined}
    \end{small}
\end{table}

Figure~\ref{fig:synthesis} shows the programs computed by  $BFGP(h_5,f_1)$. The state variables in {\em Triangular Sum} are two pointers ($a$ and $b$) and two memory registers (v[0],v[1]). The program for {\em Triangular Sum} starts with $a=0$, $b=0$, $v[0]=0$ and $v[1]=n$; first it increases $b$ by one, then iterates adding the content of $v[1]$ to $v[0]$, and decreasing $v[1]$ until it becomes $0$. The {\em Corridor} domain has a pointer $i$, which points to a register with the current location, and a constant pointer $g_i$ that accesses a memory register with the goal location. The solution moves the agent right in the corridor until the goal location is surpassed, then it moves the agent left until the goal is reached.

The {\em Reverse}, {\em Select} and {\em Find} domains have three pointers each, and one of these pointers is a constant pointer indicating the vector size. The {\em Reverse} solution is the same illustrated in Figure~\ref{fig:lists}, that was already explained above. {\em Select} uses pointer $a$ to iterate over the vector,  compares each position with the content of $b$, and assigns $a$ to $b$ if the content of $a$ is smaller than the content of $b$. The {\em Find} program requires a counter to accumulate the number of occurrences of the $\mathsf{target}$ element to be found; the {\em Find} program compares the content of $a$ with a $\mathsf{target}$ content, and if they are equal, the $\mathsf{counter}$ is increased by one.

In {\em Fibonacci} there is a vector that stores the first $n$ numbers of the Fibonacci sequence. The constant pointer $n$ indicates the vector size, and pointers $b$ and $c$ are used to iterate over the vector (pointer $c$ acts as $F_n$, while $b$ plays the role of $F_{n-1}$ and $F_{n-2}$). The solution program for {\em Gripper} only uses the left arm of the robot to pick a ball, then moves the robot to the next room, drops the ball, and moves the robot back to the initial room, repeating the entire process until the last object is processed. {\em Sorting} is solved by $BFGP(h_5,f_1)$ with a succinct but more complex implementation of the selection-sort algorithm, where the first register is used as the selected element, and that requires many back and forth iterations until pointer $i$ reaches the first element.

\begin{figure*}
\hspace{-3cm}
	\begin{scriptsize}
		\begin{subfigure}[t]{0.20\textwidth}
			\begin{lstlisting}[mathescape]
			0. inc(b)
			1. add(*a,*b)
			2. dec(*b)
			3. goto(1,$\neg$($y_z\wedge\neg y_c$))
			4. end
			
			
			
			\end{lstlisting}
		\end{subfigure}	
		\hspace{.5cm}
		\begin{subfigure}[t]{0.20\textwidth}
			\begin{lstlisting}[mathescape]
			0. right(*i)
			1. cmp(*i,*gi)
			2. goto(0,$\neg$($\neg y_z\wedge y_c$))
			3. left(*i)
			4. cmp(*i,*gi)
			5. goto(1,$\neg$($ y_z\wedge \neg y_c$))
			6. end
			
			\end{lstlisting}
		\end{subfigure}
		\hspace{.5cm}		
		\begin{subfigure}[t]{0.20\textwidth}
			\begin{lstlisting}[mathescape]
			0. set(j,tail)
			1. swap(*i,*j)
			2. dec(j)
			3. inc(i)
			4. cmp(j,i)
			5. goto(1,$\neg (\neg y_z \wedge \neg y_c)$)
			6. end
			
			\end{lstlisting}
		\end{subfigure}	
		\hspace{.5cm}
		\begin{subfigure}[t]{0.20\textwidth}
			\begin{lstlisting}[mathescape]
			0. inc(a)
			1. cmp(*b,*a)
			2. goto(4,$\neg (\neg y_z \wedge y_c)$)
			3. set(b,a)
			4. cmp(tail,a)
			5. goto(0,$\neg (y_z \wedge \neg y_c)$)
			6. end
			
			\end{lstlisting}
		\end{subfigure}
		\newline	
		\begin{subfigure}[t]{0.20\textwidth}
			\begin{lstlisting}[mathescape]
     0. cmp(*target,*a)
     1. goto(3,$\neg$($y_z\wedge\neg y_c$))
     2. inc(accumulator)
     3. inc(a)
     4. cmp(tail,a)
     5. goto(0,$\neg$($y_z\wedge\neg y_c$))
     6. end


			\end{lstlisting}
		\end{subfigure}	
		\hspace{.5cm}
		\begin{subfigure}[t]{0.20\textwidth}
			\begin{lstlisting}[mathescape]
     0. inc(c)
     1. inc(c)
     2. add(*c,*b)
     3. inc(b)
     4. add(*c,*b)
     5. cmp(n,c)
     6. goto(0,$\neg$($y_z\wedge\neg y_c$))
     7. end
     
			\end{lstlisting}
		\end{subfigure}	
		\hspace{.5cm}
		\begin{subfigure}[t]{0.20\textwidth}
			\begin{lstlisting}[mathescape]
     0. pick(*left$\text{-}$arm)
     1. moveAB()
     2. drop(*left$\text{-}$arm)
     3. inc(left$\text{-}$arm)
     4. moveBA()
     5. cmp(left$\text{-}$arm,last$\text{-}$obj)
     6. goto(0,$\neg$($y_z\wedge\neg y_c$))
     7. end

			\end{lstlisting}
		\end{subfigure}
		\hspace{.5cm}
		\begin{subfigure}[t]{0.20\textwidth}
			\begin{lstlisting}[mathescape]
     0. set(i,tail)
     1. swap(*i,*j)
     2. cmp(*j,*i)
     3. goto(6,$\neg$($\neg y_z\wedge y_c$))
     4. inc(i)
     5. goto(0,$\neg$($y_z\wedge y_c$))
     6. dec(i)
     7. goto(1,$\neg$($y_z\wedge\neg y_c$))
     8. end
			\end{lstlisting}
		\end{subfigure}	
	\end{scriptsize}
	\caption{Solutions computed by $BFGP(h_5,f_1)$ for Tr. Sum, Corridor, Reverse, Select, Find, Fibonacci, Gripper and Sorting.}
	\label{fig:synthesis}
\end{figure*}

\subsection{Validation of the GP Solutions}
Table~\ref{tab:validation} reports the CPU time, and peak memory, yield when running the solutions synthesized by $BFGP(h_5,f_1)$ on a validation set. All the solutions synthesized by $BFGP(h_5,f_1)$ were successfully validated. In the validation set the domain of the planning state variables equals the integer bounds ($10^9$). The solution for {\em Triangular Sum} is validated over $44{,}709$ instances, the last one corresponding to the $44{,}720^{th}$ term in the sequence, i.e.~the integer $999{,}961{,}560$. Fibonacci has a validation set of 33 instances, ranging from the $11^{th}$ Fibonacci term to the $43^{rd}$, i.e.~the integer $701{,}408{,}733$. The solutions for {\em Reverse}, {\em Select}, and {\em Find} domains, are validated on $50$ instances each, with vector sizes ranging from $1K$ to $50K$, and random integer elements bounded by $10^9$. {\em Corridor} validation instances range from corridor lengths of $12$ to $1{,}011$, and {\em Gripper} is validated on $1{,}000$ instances with up to $1{,}011$ balls. Validation instances in {\em Sorting} are vectors with sizes ranging from $12$ to $31$, and with random integer elements bounded by $10^9$.  The largest CPU-times and memory peaks correspond to the configuration that implements the detection of {\em infinite programs}, which requires saving states to detect whether they are revisited during execution. Skipping this mechanism allows to validate non-infinite programs faster~\cite{aguas2020generalized}.

\begin{table}
	\centering
	\begin{small}
		\begin{tabular}{|l|r||r|r||r|r|}\hline
			\textbf{Dom.}  & Inst. & Time$_\infty$ & Mem$_\infty$ & Time & Mem \\\hline
			T. Sum & 44,709 & 1,066.74 & 53MB & \textbf{574.08} & \textbf{47MB} \\  
			Corr. & 1,000 & 0.23 & 5.0MB & \textbf{0.15} & {\bf 4.7MB} \\
			Rev. & 50 & 37.96 & 5.2GB & \textbf{2.70} & \textbf{0.3GB} \\
			Sel. & 50 & 144.75 & 19.6GB & \textbf{2.29} & \textbf{33MB} \\
			Find & 50 & 114.55 & 19.6GB & \textbf{2.12} & \textbf{33MB}\\
			Fibo. & 33 & \textbf{0.00} & 4.2MB & \textbf{0.00} & \textbf{3.9MB} \\
			Grip. & 1,000 & 2.71 & 0.1GB & \textbf{1.65} & \textbf{0.1GB} \\
			Sort. & 20 & 272.06 & 15.2GB & \textbf{52.04} & \textbf{3.8MB} \\\hline
		\end{tabular}
	\end{small}
	\caption{Validation set, CPU-time (secs) and memory peak for program validation, with/out {\em infinite program} detection.}
	\label{tab:validation}
\end{table}

\subsection{Comparison with GP Compilation}
Table~\ref{tab:comparisson} compares the performance of $BFGP (h_5,f_1)$, in terms of CPU-time, with PP, which stands for the compilation-based approach for GP~\cite{segovia2019computing}, that computes planning programs with the {\sc LAMA-2011} setting (first solution) to solve the classical planning problems that result from the compilation. The {\em Gripper} domain is the only one where it is easier to compute a generalized plan following the classical planning compilation. In addition, the compilation-based approach reported no solution for the {\em Corridor} and {\em Sorting} domains. What is more,  the compilation approach (PP) is unable to perform the previous validation; off-the-shelf planner cannot handle state variables with the reported domain sizes. 

\begin{table}[]
    \centering
    \small
    \begin{tabular}{|l|c|c|} \hline%\cline{2-9}
        {\bf Domain}& PP in sec. & BFGP$(h_5,f_1)$ in sec. \\ \hline %\cline{2-5}
          %& $n$ & T (s) & $n$ & T (s) \\\hline
        Triangular Sum & 0.85 & \textbf{0.1} \\
        Corridor  & - & \textbf{4.5} \\
        Reverse  & 87.86 & \textbf{1.4} \\
        Select  & 204.20 & \textbf{80} \\
        Find  & 274.86 & \textbf{162} \\
        Fibonacci  & 3,570 & \textbf{22} \\
        Gripper  & {\bf 1} & 6.9 \\ 
        Sorting & - & {\bf 713} \\\hline
    \end{tabular}
    \caption{Computing CPU-time (secs) for solving domains in the GP compilation approach (PP) and $BFGP (h_5,f_1)$. }
    \label{tab:comparisson}
\end{table}

\section{Related Work}
This work is the first full native heuristic search approach to GP, with a search space, evaluation/heuristic functions, and search algorithm, that are specially targeted to GP. Our {\em GP as heuristic search} approach is related to the classical planning compilation for GP~\cite{segovia2019computing}. The paper showed that this compilation presented an important drawback; the size of the classical planning instance output by the compilation grows exponentially with the number and domain size of the state variables. In practice, this drawback limits the applicability of the cited compilation to planning instances of small size.

A related trend in GP represents generalized plans as lifted policies that are computed via a FOND planner~\cite{bonet2019learning,illanes2019generalized}. Instantiating a lifted policy for solving a specific classical planning instance requires a matching process over the objects in the instance. This instantiation requirement makes it difficult to define native heuristic functions for searching in the space of possible GP solutions with a classic heuristic search algorithm. On the other hand, our approach naturally suits a classic best-first search framework; the execution of a planning program on a classical planning instance is a matching-free process that enables the definition of effective evaluation/heuristic functions.

Our feature language is related to the ones used in {\em Qualitative numeric planning}~\cite{srivastava2011qualitative,bonet2019qualitative,illanes2019generalized}. {\em Qualitative numeric planning} leverages propositional variables to abstract the value of numeric state variables. In our work the value of FLAGS $Y=\{y_z,y_c\}$ depends on the last executed action. Considering that only RAM instructions update the variables in $Y$, we have an observation space of $2^{|Y|}\times 2|Z|^2$ state observations implemented with only $|Y|$ Boolean variables. The four joint values of $\{y_z,y_c\}$ can model a large space of observations, e.g.~$=\,$0, $\neq\,$0, $<0, >0, \leq 0, \geq 0$ as well as relations $=, \neq, <, >, \leq, \geq$ on variable pairs. Our state observations may also refer to pointer contents; in the sorting program of Figure~\ref{fig:lists}, {\tt\small cmp(*j,*min)} is conditioned on the value of the variables referenced by two pointers, while {\tt\small cmp(length,j)} is conditioned on the pointers themselves.

Last but not least, our GP solution representation can be understood as a program encoded in a minimalist programming language (i.e. with a lean instruction set). Furthermore, many GP benchmarks have solutions that are known to be polynomial algorithms. Previous work on generalized planning~\cite{segovia2019computing,illanes2019generalized,bonet2020high} was already aware of the connection between GP and {\em program synthesis}~\cite{gulwani2017program,lee2018accelerating,alur2018search}. We believe this work builds a stronger connection between these two closely related areas.

\section{Conclusion}
We presented the first native heuristic search approach for GP. Since we are approaching GP as a classic tree search, a wide landscape of effective techniques, coming from heuristic search and classical planning, may improve the base performance of our approach. We mention some of the more promising ones. Large open lists can be handled more effectively  splitting them in several smaller lists~\cite{Helmert:FD:JAIR06}. {\em Delayed duplicate detection} could be implemented to manage large closed lists with magnetic disk memory~\cite{korf2008linear}. Further, once a search node is cancelled (e.g.~because $h_i(\Pi,\mathcal{P})$ identified that the {\em planning program} fails on a given instance), any program equivalent to this node should also be cancelled, e.g.~any program that can be built with transpositions of the causally-independent instructions. 

{\sc BFGP} starts from the empty program but nothing prevents us from starting search from a partially specified planning program. In fact, local search approaches have already shown successful for planning~\cite{gerevini2003planning} and program synthesis~\cite{solar2009sketching,gulwani2017program}. SATPLAN planners exploit multiple-thread computing to parallelize search in solution spaces with different bounds~\cite{rintanen2012planning}. This same idea could be applied to multiple searches for GP solutions with different program sizes.  Last but not least, our goal-driven heuristic $h_5(\Pi,P_t)$ builds on top of the {\em Euclidean distance}; better estimates may be obtained by building on top of better informed planning heuristics~\cite{hoffmann2003metric,Helmert:FD:JAIR06,frances2017effective}. 

\begin{small}
\section*{Acknowledgments}
Javier Segovia-Aguas is supported by TAILOR, a project funded by EU H2020 research and innovation programme no. 952215, an ERC Advanced Grant no. 885107, and grant TIN-2015-67959-P from MINECO, Spain. Sergio Jim\'enez is supported by the {\it Ramon y Cajal} program, RYC-2015-18009, the Spanish MINECO project TIN2017-88476-C2-1-R. Anders Jonsson is partially supported by Spanish grants PID2019-108141GB-I00 and PCIN-2017-082.
\end{small}

\bibliography{generalized-aaai21}

\begin{thebibliography}{41}
\providecommand{\natexlab}[1]{#1}
\providecommand{\url}[1]{\texttt{#1}}
\providecommand{\urlprefix}{URL }
\expandafter\ifx\csname urlstyle\endcsname\relax
  \providecommand{\doi}[1]{doi:\discretionary{}{}{}#1}\else
  \providecommand{\doi}{doi:\discretionary{}{}{}\begingroup
  \urlstyle{rm}\Url}\fi

\bibitem[{Alur et~al.(2018)Alur, Singh, Fisman, and
  Solar-Lezama}]{alur2018search}
Alur, R.; Singh, R.; Fisman, D.; and Solar-Lezama, A. 2018.
\newblock Search-based program synthesis.
\newblock \emph{Communications of the ACM} 61(12): 84--93.

\bibitem[{Bonet et~al.(2020)Bonet, De~Giacomo, Geffner, Patrizi, and
  Rubin}]{bonet2020high}
Bonet, B.; De~Giacomo, G.; Geffner, H.; Patrizi, F.; and Rubin, S. 2020.
\newblock High-Level Programming via Generalized Planning and LTL Synthesis.
\newblock In \emph{KR}, volume~17, 152--161.

\bibitem[{Bonet, Frances, and Geffner(2019)}]{bonet2019learning}
Bonet, B.; Frances, G.; and Geffner, H. 2019.
\newblock Learning features and abstract actions for computing generalized
  plans.
\newblock In \emph{AAAI}, volume~33, 2703--2710.

\bibitem[{Bonet and Geffner(2001)}]{bonet2001planning}
Bonet, B.; and Geffner, H. 2001.
\newblock Planning as heuristic search.
\newblock \emph{Artificial Intelligence} 129: 5--33.

\bibitem[{Bonet and Geffner(2020)}]{bonet2019qualitative}
Bonet, B.; and Geffner, H. 2020.
\newblock Qualitative Numeric Planning: Reductions and Complexity.
\newblock \emph{JAIR} 69: 923--961.

\bibitem[{Bonet, Palacios, and Geffner(2010)}]{Geffner:FSM:AAAI10}
Bonet, B.; Palacios, H.; and Geffner, H. 2010.
\newblock Automatic Derivation of Finite-State Machines for Behavior Control.
\newblock In \emph{AAAI}, volume~24.

\bibitem[{Boolos, Burgess, and Jeffrey(2002)}]{boolos2002computability}
Boolos, G.~S.; Burgess, J.~P.; and Jeffrey, R.~C. 2002.
\newblock \emph{Computability and logic}.
\newblock Cambridge university press.

\bibitem[{Dandamudi(2005)}]{dandamudi2005installing}
Dandamudi, S.~P. 2005.
\newblock Installing and using nasm.
\newblock \emph{Guide to Assembly Language Programming in Linux} .

\bibitem[{Franc{\`e}s et~al.(2017)}]{frances2017effective}
Franc{\`e}s, G.; et~al. 2017.
\newblock \emph{Effective planning with expressive languages}.
\newblock Ph.D. thesis, Universitat Pompeu Fabra.

\bibitem[{Franc{\`e}s~Medina et~al.(2019)Franc{\`e}s~Medina, Corr{\^e}a,
  Geissmann, and Pommerening}]{frances2019generalized}
Franc{\`e}s~Medina, G.; Corr{\^e}a, A.~B.; Geissmann, C.; and Pommerening, F.
  2019.
\newblock Generalized potential heuristics for classical planning.
\newblock In \emph{IJCAI}.

\bibitem[{Gerevini, Saetti, and Serina(2003)}]{gerevini2003planning}
Gerevini, A.; Saetti, A.; and Serina, I. 2003.
\newblock Planning through stochastic local search and temporal action graphs
  in LPG.
\newblock \emph{JAIR} 20: 239--290.

\bibitem[{Gulwani et~al.(2017)Gulwani, Polozov, Singh
  et~al.}]{gulwani2017program}
Gulwani, S.; Polozov, O.; Singh, R.; et~al. 2017.
\newblock Program synthesis.
\newblock \emph{Foundations and Trends in Programming Languages} .

\bibitem[{Helmert(2006)}]{Helmert:FD:JAIR06}
Helmert, M. 2006.
\newblock {The Fast Downward Planning System}.
\newblock \emph{Journal of Artificial Intelligence Research} 26: 191--246.

\bibitem[{Hoffmann(2001)}]{hoffmann2001ff}
Hoffmann, J. 2001.
\newblock FF: The fast-forward planning system.
\newblock \emph{AI magazine} 22(3): 57--57.

\bibitem[{Hoffmann(2003)}]{hoffmann2003metric}
Hoffmann, J. 2003.
\newblock The Metric-FF Planning System: Translating "Ignoring Delete Lists" to
  Numeric State Variables.
\newblock \emph{JAIR} 20: 291--341.

\bibitem[{Hu and De~Giacomo(2011)}]{hu2011generalized}
Hu, Y.; and De~Giacomo, G. 2011.
\newblock Generalized planning: Synthesizing plans that work for multiple
  environments.
\newblock In \emph{IJCAI}, volume~22, 918--923.

\bibitem[{Hu and {De Giacomo}(2013)}]{Giacomo:FSM:ICAPS13}
Hu, Y.; and {De Giacomo}, G. 2013.
\newblock A Generic Technique for Synthesizing Bounded Finite-State
  Controllers.
\newblock In \emph{ICAPS}.

\bibitem[{Hu and Levesque(2011)}]{Levesque:GPlanning:IJCAI11}
Hu, Y.; and Levesque, H.~J. 2011.
\newblock A Correctness Result for Reasoning about One-Dimensional Planning
  Problems.
\newblock In \emph{IJCAI}, 2638--2643.

\bibitem[{Illanes and McIlraith(2019)}]{illanes2019generalized}
Illanes, L.; and McIlraith, S.~A. 2019.
\newblock Generalized planning via abstraction: arbitrary numbers of objects.
\newblock In \emph{AAAI}, volume~33, 7610--7618.

\bibitem[{Jim{\'e}nez, Segovia-Aguas, and Jonsson(2019)}]{jimenez2019review}
Jim{\'e}nez, S.; Segovia-Aguas, J.; and Jonsson, A. 2019.
\newblock A review of generalized planning.
\newblock \emph{KER} 34: 1--28.

\bibitem[{Korf(2008)}]{korf2008linear}
Korf, R.~E. 2008.
\newblock Linear-time disk-based implicit graph search.
\newblock \emph{Journal of the ACM} 55(6): 1--40.

\bibitem[{Korf et~al.(2005)Korf, Zhang, Thayer, and Hohwald}]{korf2005frontier}
Korf, R.~E.; Zhang, W.; Thayer, I.; and Hohwald, H. 2005.
\newblock Frontier search.
\newblock \emph{Journal of the ACM} 52(5): 715--748.

\bibitem[{Lee et~al.(2018)Lee, Heo, Alur, and Naik}]{lee2018accelerating}
Lee, W.; Heo, K.; Alur, R.; and Naik, M. 2018.
\newblock Accelerating search-based program synthesis using learned
  probabilistic models.
\newblock \emph{ACM SIGPLAN Notices} .

\bibitem[{Lipovetzky and Geffner(2017)}]{lipovetzky2017best}
Lipovetzky, N.; and Geffner, H. 2017.
\newblock Best-first width search: Exploration and exploitation in classical
  planning.
\newblock In \emph{AAAI}, volume~31.

\bibitem[{Lotinac et~al.(2016)Lotinac, Segovia-Aguas, Jim{\'e}nez, and
  Jonsson}]{lotinac2016automatic}
Lotinac, D.; Segovia-Aguas, J.; Jim{\'e}nez, S.; and Jonsson, A. 2016.
\newblock Automatic generation of high-level state features for generalized
  planning.
\newblock In \emph{IJCAI}, 3199--3205.

\bibitem[{Mart{\'\i}n and Geffner(2004)}]{Geffner:Gpolicies:AppliedI04}
Mart{\'\i}n, M.; and Geffner, H. 2004.
\newblock Learning generalized policies from planning examples using concept
  languages.
\newblock \emph{Applied Intelligence} 20: 9--19.

\bibitem[{Minsky(1961)}]{minsky1961recursive}
Minsky, M.~L. 1961.
\newblock Recursive unsolvability of Post's problem of ``tag'' and other topics
  in theory of Turing machines.
\newblock \emph{Annals of Mathematics} .

\bibitem[{Nau et~al.(2003)Nau, Au, Ilghami, Kuter, Murdock, Wu, and
  Yaman}]{nau:shop2:JAIR03}
Nau, D.~S.; Au, T.-C.; Ilghami, O.; Kuter, U.; Murdock, J.~W.; Wu, D.; and
  Yaman, F. 2003.
\newblock SHOP2: An HTN planning system.
\newblock \emph{JAIR} 20: 379--404.

\bibitem[{Ramirez and Geffner(2016)}]{ramirez2016heuristics}
Ramirez, M.; and Geffner, H. 2016.
\newblock Heuristics for Planning, Plan Recognition and Parsing.
\newblock \emph{arXiv:1605.05807} .

\bibitem[{Richter and Westphal(2010)}]{richter2010lama}
Richter, S.; and Westphal, M. 2010.
\newblock The LAMA planner: Guiding cost-based anytime planning with landmarks.
\newblock \emph{JAIR} 39: 127--177.

\bibitem[{Rintanen(2012)}]{rintanen2012planning}
Rintanen, J. 2012.
\newblock Planning as satisfiability: Heuristics.
\newblock \emph{Artificial intelligence} 193: 45--86.

\bibitem[{Segovia-Aguas, Jim{\'e}nez, and
  Jonsson(2016)}]{javi-Gplanning-IJCAI16}
Segovia-Aguas, J.; Jim{\'e}nez, S.; and Jonsson, A. 2016.
\newblock Hierarchical Finite State Controllers for Generalized Planning.
\newblock In \emph{IJCAI}, 3235--3241.

\bibitem[{Segovia-Aguas, Jim{\'e}nez, and
  Jonsson(2017)}]{segovia2017generating}
Segovia-Aguas, J.; Jim{\'e}nez, S.; and Jonsson, A. 2017.
\newblock Generating context-free grammars using classical planning.
\newblock In \emph{IJCAI}.

\bibitem[{Segovia-Aguas, Jim{\'e}nez, and Jonsson(2019)}]{segovia2019computing}
Segovia-Aguas, J.; Jim{\'e}nez, S.; and Jonsson, A. 2019.
\newblock Computing programs for generalized planning using a classical
  planner.
\newblock \emph{Artificial Intelligence} 272: 52--85.

\bibitem[{Segovia-Aguas, Jim{\'e}nez, and Jonsson(2020)}]{aguas2020generalized}
Segovia-Aguas, J.; Jim{\'e}nez, S.; and Jonsson, A. 2020.
\newblock Generalized Planning with Positive and Negative Examples.
\newblock In \emph{AAAI}, volume~34, 9949--9956.

\bibitem[{Segovia-Aguas, Jiménez, and
  Jonsson(2021)}]{segovia_aguas_javier_2021_4600833}
Segovia-Aguas, J.; Jiménez, S.; and Jonsson, A. 2021.
\newblock Best First Generalized Planning.
\newblock \doi{10.5281/zenodo.4600833}.
\newblock \url{https://doi.org/10.5281/zenodo.4600833}.

\bibitem[{Siddharth et~al.(2011)Siddharth, Neil, Shlomo, and
  Tianjiao}]{Zilberstein:Gplanning:icaps11}
Siddharth, S.; Neil, I.; Shlomo, Z.; and Tianjiao, Z. 2011.
\newblock Directed Search for Generalized Plans Using Classical Planners.
\newblock In \emph{ICAPS}, volume~21.

\bibitem[{Skiena(1998)}]{skiena1998algorithm}
Skiena, S.~S. 1998.
\newblock \emph{The algorithm design manual: Text}, volume~1.
\newblock Springer Science \& Business Media.

\bibitem[{Solar-Lezama(2009)}]{solar2009sketching}
Solar-Lezama, A. 2009.
\newblock The sketching approach to program synthesis.
\newblock In \emph{Asian Symposium on Programming Languages and Systems}.

\bibitem[{Srivastava et~al.(2011)Srivastava, Zilberstein, Immerman, and
  Geffner}]{srivastava2011qualitative}
Srivastava, S.; Zilberstein, S.; Immerman, N.; and Geffner, H. 2011.
\newblock Qualitative numeric planning.
\newblock In \emph{AAAI}, volume~25.

\bibitem[{Winner and Veloso(2003)}]{Winner03distill:learning}
Winner, E.; and Veloso, M. 2003.
\newblock DISTILL: Learning Domain-Specific Planners by Example.
\newblock In \emph{ICML}, 800--807.

\end{thebibliography}
\end{document}